\documentclass[journal]{IEEEtran}

\ifCLASSINFOpdf
\else
   \usepackage[dvips]{graphicx}
\fi
\usepackage{url}
\usepackage{graphicx}
\usepackage{xcolor}
\usepackage{adjustbox}
\usepackage{multicol}
\usepackage{multirow}
\usepackage{adjustbox}
\usepackage{wrapfig}
\usepackage{pgfplotstable}
\usepackage{multirow}
\usepackage{wrapfig}
\usepackage{algorithm}
\usepackage{algpseudocode}
\usepackage{hyperref}

\usepackage{subcaption}

\usepackage{amsfonts}
\usepackage{amsmath}
\usepackage{amssymb}
\usepackage{amsthm}

\newcommand{\risk}{R}
\newcommand{\truth}{x} 
\newcommand{\observation}{\widetilde{x}}
\newcommand{\observationset}{\mathcal{Z}}
\newcommand{\target}{y}
\newcommand{\detector}{\delta}

\newcommand{\modelparams}{\theta}
\newcommand{\modelparamsaug}{{\theta_0, \theta_1}}

\newcommand{\model}{f}

\newcommand{\expect}{\mathbb{E}}

\newcommand{\var}{\mathbb{V}}
\newcommand{\loss}{\ell}
\newcommand{\inferredtruth}{\widehat{\truth}}
\newcommand{\prot}{\text{p}}
\newcommand{\unprot}{\text{u}}
\newcommand{\attackvar}{a}
\newcommand{\unprotmatrix}{M}
\newcommand{\optimalpredweights}{ \var \left[ x_u, x_p \right] \var \left[x_p \right]^{-1}}

\theoremstyle{definition}
\newtheorem{definition}{Definition}

\newtheorem{lemma}{Lemma}
\newtheorem{proposition}{Proposition}

\usepackage{tikz}
\usepackage{pgfplots}
\usepackage{graphicx}
\usepackage{xcolor, soul}
\usepackage{mdframed}
\usepackage{lipsum} 
\pgfplotsset{compat=1.14}
\usepgfplotslibrary{external}
\usetikzlibrary{arrows.meta,calc,chains,shapes.geometric}
\usepgfplotslibrary{fillbetween}
\usepackage{tikz}
\usetikzlibrary{positioning}

\pgfplotsset{compat=1.14}
\usepgfplotslibrary{external}
\usetikzlibrary{arrows.meta,calc,chains,shapes.geometric}

\definecolor{cb-black}      {RGB}{  0,   0,   0}
\definecolor{cb-dark-blue-green} {RGB}{  0,  073,  073}
\definecolor{cb-blue-green} {RGB}{  0,  150,  150}
\definecolor{cb-green-sea}  {RGB}{  0, 146, 146}
\definecolor{cb-rose}       {RGB}{255, 109, 182}
\definecolor{cb-dark-rose}       {RGB}{127, 54, 91}
\definecolor{cb-salmon-pink}{RGB}{255, 182, 119}
\definecolor{cb-purple}     {RGB}{ 73,   0, 146}
\definecolor{cb-blue}       {RGB}{ 0, 109, 219}
\definecolor{cb-lilac}      {RGB}{182, 109, 255}
\definecolor{cb-blue-sky}   {RGB}{109, 182, 255}
\definecolor{cb-blue-light} {RGB}{182, 219, 255}
\definecolor{cb-burgundy}   {RGB}{146,   0,   0}
\definecolor{cb-brown}      {RGB}{146,  73,   0}
\definecolor{cb-clay}       {RGB}{219, 209,   0}
\definecolor{cb-green-lime} {RGB}{ 36, 255,  36}
\definecolor{cb-yellow}     {RGB}{255, 255, 109}

\begin{document}

\title{Adversarial Training of Linear Models \\ under Stealthy Attacks}

\author{Lovisa Eriksson, Dave Zachariah, and André M. H. Teixeira
\thanks{The authors are with Department of Information Technology, Uppsala University, Uppsala, Sweden. This work is supported by the Swedish Research Council under grants 2023-05234, 2021-05022 and 2024-03903; by the Swedish Foundation for Strategic Research; and by the Knut and Alice Wallenberg Foundation.}}

{}
\maketitle

\begin{abstract}
 Predictive models are widely used in many fields, but are vulnerable to false data injection attacks.
 To address this, detection schemes and adversarial training have been proposed, but such approaches lack guarantees against stealthy attacks. 
 We therefore propose a detector-based switched model, in which optimal attack strategies are stealthy. For linear prediction models, we derive a convex formulation of the resulting adversarial risk. The model incorporates protected features and introduces a hyperparameter modelling attack probability, enabling an explicit performance trade-off between clean and attacked data regimes.
 Numerical simulations on real and synthetic data show improved performance on partially attacked data, even for misspecified attack probabilities.

\end{abstract}


\IEEEpeerreviewmaketitle

\section{Introduction}
\label{section:introduction}
A basic idea for most modern machine learning algorithms is that a predictive model can be optimized on a set of past data such that it performs well on new data. However, in reality, the new data may not be drawn from the same distribution as the past data, due to natural variations \cite{Huber1964robust}, or false data injection attacks \cite{Liang2017review}. Applying a model to data subject to attacks presents a particular challenge since the new data is intentionally crafted to deteriorate the model performance \cite{Ribeiro2023overparametrized}.

False data injection attacks have been extensively studied, both in machine learning \cite{Ribeiro2023overparametrized, Carlini2017adversarial, Goodfellow2015explaining} and adjacent fields such as control theory \cite{Liang2017review}. To make the problem feasible, different assumptions are put on the attacker. Existing work either imposes a limit on the size of attack perturbations \cite{Ribeiro2023overparametrized, Goodfellow2015explaining} or imposes stealthiness \cite{Urbina2016limiting, Mao2020novel}. Classical robust statistics studies estimators under bounded contamination and worst-case deviations from nominal assumptions \cite{Huber1964robust}, but without analytically characterizing the resulting performance trade-off's across regimes.
Common defence mechanisms include adversarial training and detection schemes. Adversarial training \cite{Goodfellow2015explaining}, \cite{zhao2022adversarial} improves robustness through the generation of static perturbed examples, thus limiting the defence against unforeseen and adapting attacks. Detection schemes against adversarial examples has been shown infective against stronger, tailored attacks \cite{Carlini2017adversarial}.



To address these gaps simultaneously, we consider
the use of a predictive model that switches mode depending on whether an attack detector triggers an alarm or not. Importantly, the attack can be stealthy and we consider the possibility that some features of the new data can be protected. Formal results are derived for linear models. Our main contributions are (i) a  method that makes explicit trade-offs between accuracy on clean and attacked data, 
(ii) deriving adversarial stealthiness as a consequence of the proposed prediction model,
and (iii) providing explicit convex formulations for adversarial training of linear models under stealthy attacks. Numerical experiments illustrating the performance use both synthetic and real data.

\textit{Notation:} Throughout the paper we use $\| \cdot \|$ to denote the euclidean norm, $\expect[x]$ to denote expectation of $x$ and $\var[x, x']$ is the covariance between $x$ and $x'$, with $\var[x] \equiv \var[x,x]$. The natural logarithm is denoted $\log$ and the transpose of a vector $x$ is denoted $x^\top$. We use $p(x)$ to denote the probability density function of $x$, and $p(x=k)$ to denote the probability of event $k$.

\textit{The code is available at GitHub:}  \\\href{https://github.com/LovisaLuna/code_Adversarial-Training-of-Linear-Models-under-Stealthy-Attacks}{LovisaLuna/code\_Adversarial-Training-of-Linear-Models-under-Stealthy-Attacks}



\section{Problem Formulation}
\label{section:problem_formulation}

Let $(x,y) \in \mathcal{X} \times \mathcal{Y} \subseteq \mathbb{R}^d \times \mathbb{R}$ denote a pair of features and outcomes of interest. Let $\model_{\modelparams}(\truth)$ be a model parametrized by $\modelparams$ that predicts $y$ from $x$ with loss $\loss_\modelparams(x,y)$. Specifically, we study two cases: (i) regression $\mathcal{Y}= \mathbb{R}$ using the standard linear model with squared error loss
\begin{equation}
    \label{eq:linearregression}
    \model_{\modelparams}(\truth) = \modelparams^\top x ,\quad \loss_\theta(x, y) = | y - \model_\modelparams(x) |^2, \text{ and}
\end{equation}
(ii) binary classification  $\mathcal{Y}= \{-1, 1 \}$ using logistic regression with logistic loss
\begin{equation}
    \label{eq:logisticregression}
    \model_{\modelparams}(\truth) = \frac{1 - e^{-\modelparams^\top x}}{1 + e^{-\modelparams^\top x}}, \quad  \loss_{\modelparams}(\truth, y) = \log\left(1 + e^{-y \modelparams^\top x}\right). 
\end{equation}

During \emph{training}, we observe $n$ samples $\{(x_i,y_i)\}^n_{i=1}$ drawn independently from $p(x,y)$. At \emph{test} time, however, we observe the features $\observation$ which could be manipulated by an attack, indicated by an unobserved variable $a \in \{0,1 \}$ that occurs with probability $\gamma \in [0,1]$. We assume that features are partitioned into \emph{protected} and \emph{unprotected} components, which we express as:
\begin{equation}
    \truth = \begin{bmatrix} \truth^\top_\prot & \truth^\top_\unprot\end{bmatrix}^\top \text{ and } \observation = \begin{bmatrix} \truth^\top_\prot & \observation^\top_\unprot \end{bmatrix}^\top.
\label{eq:features}
\end{equation}

\begin{definition}
    An attack event $a=1$ is \emph{adversarial} if the observation follows
    \begin{equation}
    \label{eq:observation}
     \observation = \begin{cases}
    x, &{\text{if }} a=0,\\
    \arg \max_{z \in \observationset(\truth_\prot)} \loss_\modelparams(z, \target), &{\text{if }} a=1,
\end{cases}
\end{equation}
where $\observationset(\truth_\prot) \subseteq \{ z \in \mathcal{X}: z_p = x_p \}$. 
That is, the unprotected features in $\observation$ are designed to maximize the loss of a given model $\model_{\modelparams}$.
\end{definition}

\begin{wrapfigure}{l}{0.2\textwidth}
  \begin{center}
    \tikzset{
observed/.style = {align=center, draw, fill=white, circle, minimum height=2em, minimum width=2em, align=center},
latent/.style = {align=center, draw, dashed, fill=white, circle, minimum height=2em, minimum width=2em, align=center},
}

\begin{tikzpicture}[auto, node distance=2cm,->]
\useasboundingbox (1.8,-1.8) rectangle (-0.2,-0.2);
    
    \node [latent, name=truth] (truth) {$x$};
    \node [latent, name=attack, below of=truth, node distance = 2cm] (attack) {$a$};
    \node [latent, name=target, right of=truth, node distance = 2cm] (target) {$y$};
    \node [observed, name=observation, right of=attack, node distance = 2cm] (observation) {$\widetilde{x}$};

    \draw[->] (truth) -- (observation);
    \draw[->] (attack) -- (observation);
    \draw[->] (target) -- (observation);
    \draw[-] (truth) -- (target);

\end{tikzpicture}
  \end{center}
\end{wrapfigure}
Under this attack, we can express the process that generates the observation $\observation$ using an acyclic graph, illustrated to the left. The process induces a (marginal) distribution $p(\observation, y)$ at test time.

\begin{definition}
A (deterministic) anomaly \emph{detector} is a function
$\detector(\observation) : \mathbb{R}^d \rightarrow \{ 0,1\}$ and has a \emph{false alarm rate} $p(\detector(\widetilde{x})=1|a=0)$.
For notational brevity, the detector output will be written as $\detector = \detector(\observation)$ whenever it is unambiguous.
\end{definition}

The detector splits the feature space $\mathcal{X}$ into two parts depending on whether $\detector=0$ or $\detector=1$. 
Therefore, we propose a switched model approach that uses a \textit{nominal model} $f_{\theta_0}$ and a \textit{recovery model} $f_{\theta_1}$, depending on $\delta$.
We can express the above approach using the switched model
\begin{equation}
\label{eq:modelswitching}
    \model_{\modelparams_0, \modelparams_1}(\observation, \detector) = \begin{cases}
        \model_{\modelparams_0}(\observation), &\text{ if } \detector = 0 \\
        \model_{\modelparams_1}(\observation), &\text{ if } \detector = 1.
    \end{cases}
\end{equation}
The loss of the switched model can then be expressed as
\begin{equation}
    \label{eq:switched_loss}
    \loss_{\modelparams_0, \modelparams_1}(\observation, \target) := \detector(\observation)\loss_{\modelparams_1}(\observation, \target) + (1-\detector(\observation))\loss_{\modelparams_0}(\observation, \target),
\end{equation}
such that its expected loss, or \emph{risk}, is 
\begin{equation}
\label{eq:switched_risk}
    \risk(\modelparams_0, \modelparams_1) = \expect[\loss_{\modelparams_0, \modelparams_1}(\observation, \target)].
\end{equation}
It should be noted that minimizing the risk of the switched model can be no worse than using only one of the models $\model_{\modelparams_0}$ and $\model_{\modelparams_1}$, i.e.,
 \begin{equation}
 \label{eq:switchbetter}
     \min_{\modelparams_0, \modelparams_1} \risk(\modelparams_0, \modelparams_1) \leq \min_{\modelparams_0=\modelparams_1} \risk(\modelparams_0, \modelparams_1) = \min_{\modelparams}\expect[\loss_{\modelparams}(\observation, \target)].
 \end{equation}

However, this does not guarantee that attacked data is handled by the recovery model. If an attacker has complete knowledge of the detector function $\detector(\cdot)$, an attack event can be designed to evade detection, i.e. to be stealthy. Assuming full model knowledge of the adversary is a standard assumption in security to make sure the security of proposed methods rely only on theoretical principles, rather than lack of equipment or knowledge of the adversary \cite{Shannon1949communication}.

\begin{definition}
\label{def:stealthy}
    An adversarial attack event $a=1$ is \emph{stealthy} w.r.t. detector $\delta$ if the latter misses the attack with probability $p(\detector(\widetilde{x})=0 |\attackvar=1)=1$.  This occurs when the set of feasible attacks is further restricted to exploit the detector, by setting
    \begin{equation}
    \label{eq:stealthyRestriction}
    \observationset(\truth_\prot) \subseteq \{ z \in \mathcal{X}: z_p = x_p \} \cap \{ z \in \mathcal{X} : \detector(z) = 0 \}.
    \end{equation}
\end{definition}
   
The goal of the paper is to learn a risk-minimizing switched model $\model_\modelparamsaug$, using clean training data, that is robust against adversarial and stealthy attacks at test time, given a detector with false alarm rate no greater than $\Bar{\alpha}$.

\section{Theoretical results}
\label{section:theoretical_results}

We now show that under certain conditions on the recovery model $\model_{\modelparams_1}$, the attack  \eqref{eq:observation} will be essentially equivalent to an \emph{stealthy} adversarial attack. We formalise the required conditions as $\model_{\modelparams_1}$ being \textit{secure}.

\begin{definition}
\label{def:secure}
The recovery model $\model_{\modelparams_1}(x)$ is said to be \emph{secure w.r.t $\detector$ and $f_{\theta_0}$} if there exists an  imputation function $g(\truth_\prot) = \inferredtruth$ such that, for all $\tilde{x} = \begin{bmatrix}
    x_p^\top  & \tilde{x}_u^\top
\end{bmatrix}^\top$ ,
    \begin{equation}
    \label{eq:secure_def}
        \model_{\modelparams_1}(\observation) = \model_{\modelparams_0}(g(x_p)) = \model_{\modelparams_0}(\inferredtruth) \quad \text{and} \quad  \detector(\inferredtruth)=0.
    \end{equation}
\end{definition}

Thus the secure $\model_{\modelparams_1}$  imputes $\observation_\unprot$ and reverts to the prediction of $\model_{\modelparams_0}$, rather than taking the unprotected features at face value. With a secure model, an attack event can never cause more harm when detected.

\begin{proposition}
\label{prop:force-stealthy}
    Let $\model_{\modelparams_0, \modelparams_1}(\observation, \detector)$ be a model \eqref{eq:modelswitching} such that $\model_{\modelparams_1}$ is secure w.r.t $\delta$ and $f_{\theta_0}$. Then, any adversarial attack event against this model yields the \emph{same risk} as a stealthy adversarial attack. That is, letting $\mathcal{Z}(x_p) = \{z \in \mathcal{X}:  z_p = x_p\}$,
    \begin{equation}
        \max_{z \in \observationset(\truth_\prot)}\loss_{\modelparams_0, \modelparams_1}(z, y) \equiv \max_{z \in \observationset(\truth_\prot), \detector(z) = 0}\loss_{\modelparams_0, \modelparams_1}(z, y)
    \end{equation}
\end{proposition}
\begin{proof}
    Note that using the constraint $\detector(z) = 0$,
    \begin{equation}
        \max_{z \in \observationset(\truth_\prot)}\loss_{\modelparams_0, \modelparams_1}(z, y) \geq \max_{z \in \observationset(\truth_\prot), \detector(z) = 0}\loss_{\modelparams_0, \modelparams_1}(z, y).
    \end{equation}
    It is then sufficient to prove that
    \begin{equation}
    \label{eq:prop1_leq}
        \max_{z \in \observationset(\truth_\prot)}\loss_{\modelparams_0, \modelparams_1}(z, y) \leq \max_{z \in \observationset(\truth_\prot), \detector(z) = 0}\loss_{\modelparams_0, \modelparams_1}(z, y).
    \end{equation}
    Let $x^* = \arg \max_{z \in \observationset(\truth_\prot)}\loss_{\modelparams_0, \modelparams_1}(z, y)$. When $\detector(x^*)=0$, \eqref{eq:prop1_leq} holds trivially. Now assume $\detector(x^*)=1$. Then, by security of $\model_{\modelparams_1}$ as per Definition \ref{def:secure}, there exists a mapping $g(x_p)=\inferredtruth$ such that $\detector(\inferredtruth)=0$ and
    \begin{equation}
        \begin{split}
            \max_{z \in \observationset(\truth_\prot)}&\loss_{\modelparams_0, \modelparams_1}(z, y) = \loss_{\modelparams_0, \modelparams_1}(x^*, y) \\ &= \loss_{\modelparams_1}(x^*, y) = \loss_{\modelparams_0}(\inferredtruth, y) \\ &= \loss_{\modelparams_0, \modelparams_1}(\inferredtruth, y) \leq \max_{z \in \observationset(\truth_\prot), \detector(z) = 0}\loss_{\modelparams_0, \modelparams_1}(z, y).
        \end{split}
    \end{equation}
\end{proof}
Since we will use a secure model, we can from here on assume that all adversarial attacks are stealthy.
%

For the rest of this section, as well as for the numerical experiments, we consider the imputation function to be 
\begin{equation}
\label{eq:optimallinearpredictor}
    g(x_p) =\begin{bmatrix}
        x_p \\ \inferredtruth_u 
    \end{bmatrix} = \begin{bmatrix}
        I \\ \optimalpredweights
    \end{bmatrix} x_p
\end{equation}
where $\inferredtruth_u$ is the optimal linear predictor \cite{Kailath2000linear},
and use the energy detector
\begin{equation}
\label{eq:energydetector}
    \detector(\observation) = \begin{cases}
        0, &\text{ if }\| \Sigma^{-\frac{1}{2}}\left( \observation_{\unprot} - \inferredtruth_\unprot \right)\|_2^2 \leq \tau_{\Bar{\alpha}} \\
        1, &\text{ otherwise} 
    \end{cases},
\end{equation}
where $\inferredtruth_\unprot$ from \eqref{eq:optimallinearpredictor}, $\Bar{\alpha}$ a desired upper bound on the false alarm rate, and
\begin{equation}
    \label{eq:sigma}
    \Sigma = \var[x_u] - \var[x_u, x_p] \var[x_p]^{-1} \var[x_u, x_p]^\top.
\end{equation}
is the error covariance of $\inferredtruth_u$ in \eqref{eq:optimallinearpredictor}.

The threshold $\tau_{\Bar \alpha}$ can be chosen to guarantee a set false alarm rate in different ways: (i) For a general distribution, 
$
    \tau_{\Bar \alpha} = \frac{d_u}{\Bar \alpha}
$
guarantees a false alarm bounded by $\Bar \alpha$, through the Chebychev inequality \cite{Dave2020robust}; 
(ii) for a Gaussian distribution,
the detector corresponds to a chi-test and the exact false alarm rate of a given $\tau$ can thus be determined.


Using this detector together with linear or logistic regression in \eqref{eq:modelswitching}, we show that there is a choice of $\modelparams_1$ that guarantees $f_{\modelparams_1}$ to be secure w.r.t. $\detector$ and $f_{\theta_0}$. 
\begin{lemma}
\label{lemma:secure}
    Consider the switched model $\model_{\modelparams_0, \modelparams_1}$ \eqref{eq:modelswitching}, with anomaly detector $\detector$ \eqref{eq:energydetector}. Then, letting
    \begin{equation}
    \label{eq:lemma}
        \modelparams_1^\top \equiv \modelparams_0^\top \begin{bmatrix} I & 0 \\ \optimalpredweights & 0
        \end{bmatrix}
    \end{equation}
    ensures that $\model_{\modelparams_1}$ is secure w.r.t. $\delta$ and $f_{\theta_0}$.
\end{lemma}
\begin{proof}Using $g(x_p) = \begin{bmatrix}
    x_p^\top & \inferredtruth_\unprot
\end{bmatrix}^\top$ in \eqref{eq:optimallinearpredictor} with a linear regression model, we have
\begin{align}
    \model_{\modelparams_1}(\observation) &= \modelparams_0^\top \begin{bmatrix} I & 0 \\ \optimalpredweights & 0
        \end{bmatrix} \begin{bmatrix}
            x_p \\ \observation_u
        \end{bmatrix}
        \\
        &= \modelparams_0^\top g(x_p) = \model_{\modelparams_0}(g(x_p)),
\end{align}
and $\detector(\inferredtruth)=0$, thus satisfying the conditions in Definition \ref{def:secure}. A corresponding argument can be made for a logistic model. 
\end{proof}

Now, we show that the risk \eqref{eq:switched_risk} under attack is convex, guaranteeing that the global optima can be found by numerical optimization.
For notational convenience, let $M = \begin{bmatrix}
    0 & I
\end{bmatrix}^\top \Sigma^{\frac{1}{2}} $. 

\begin{proposition}
\label{prop:stealthyrisk_lin}
    Consider the linear regression model \eqref{eq:linearregression} with anomaly detector \eqref{eq:energydetector}. 
    The risk 
    \eqref{eq:switched_risk} under stealthy adversarial attacks, with attack probability $\gamma$, is 

\begin{equation}
\begin{split}
&\risk(\modelparams_0, \modelparams_1)=\gamma \expect\left[\left(\left|\target - \modelparams_0^\top g(x_p) \right| + \sqrt{\tau_{\Bar \alpha}} \left \|  \modelparams_0^\top M\right \| \right) ^2\right] \\
&+ (1-\gamma)\expect\big[\detector\|\target - \modelparams^\top_1\truth\|^2 + (1 -\detector)\|\target - \modelparams^\top_0\truth\|^2\big].
\end{split}
\label{eq:risk_stealthy_linear_regression}
\end{equation}
  
 The risk is a convex function of the model parameters $\modelparams_0$ and $\modelparams_1$.
\end{proposition}

\begin{proof}
The expectation \eqref{eq:switched_risk} can be divided into two cases depending on the attack event $\attackvar$, so that
\begin{align*}
    &\risk(\modelparams_0, \modelparams_1) \\ &= \gamma \expect[\loss_{\modelparams_0, \modelparams_1}(\observation, \target)|a=1] + (1 -\gamma) \expect[\loss_{\modelparams_0, \modelparams_1}(\observation, \target)|a=0] \\
    &= \gamma \expect[\max_{z \in \observationset(\observation_\prot), \detector(z)=0}\loss_{\modelparams_0, \modelparams_1}(z, \target)] + (1 -\gamma) \expect[\loss_{\modelparams_0, \modelparams_1}(\truth, \target)]. 
\end{align*}
By \eqref{eq:switched_loss}, the second term becomes
\begin{equation}
    (1-\gamma)\expect[\detector\|\target - \modelparams^\top_1\truth\|^2 + (1 -\detector)\|\target - \modelparams^\top_0\truth\|^2].
\end{equation}

Now, we can focus on the first term. Since the adversary is stealthy,
\begin{equation}
    \max_{z \in \observationset(\observation_\prot), \detector(z)=0}\loss_{\modelparams_0, \modelparams_1}(z, \target) = \max_{z \in \observationset(\observation_\prot), \detector(z)=0}\loss_{\modelparams_0}(z, \target),
\end{equation}
and plugging in the squared error loss this becomes
    \begin{align}
        & \max_{z \in \observationset(\observation_\prot), \detector(z) = 0} \lVert y-\modelparams_0^\top z\rVert_2^2 ] \\
         = &\max_{\detector(z) = 0} \lVert y- \modelparams_{0,\prot}^\top \truth_\prot -  \modelparams_{0,\unprot}^\top z_u \rVert_2^2  ,
    \end{align}
    where $\modelparams_0 = \begin{bmatrix}\modelparams_{0,\prot} &\modelparams_{0,\unprot}  \end{bmatrix}^\top$, and we can rewrite
    \begin{equation}
    \label{eq:proof-parameterChange}
        \modelparams_{0,\unprot}^\top z_u = \modelparams_{0, \unprot}^\top (\Sigma^{\frac{1}{2}}\Sigma^{-\frac{1}{2}}(z_\unprot-\inferredtruth_\unprot) + \inferredtruth_\unprot).
    \end{equation}

    Now, letting $\xi = \Sigma^{-\frac{1}{2}}(\observation_\unprot-\inferredtruth_\unprot)$, we get
    \begin{align}
        \max_{\|\xi\|_2 \leq \sqrt{\tau_{\Bar \alpha}}} \| y - \modelparams_{0,\prot}^\top \truth_\prot - \modelparams_{0, \unprot}^\top \Sigma^{\frac{1}{2}}\xi - \modelparams_{0, \unprot}^\top \inferredtruth_\unprot \|_2^2
    \end{align}
    By collecting the constant terms in $c = y - \modelparams_{0,\prot}^\top \truth_\prot - \modelparams_{0, \unprot}^\top \inferredtruth_\unprot$, the expression can be rewritten as
    \begin{align}
    c^2 + \max_{\|\xi\|_2 \leq \sqrt{\tau_{\Bar \alpha}}}(\modelparams_{0, \unprot}^\top \Sigma^{\frac{1}{2}}\xi)^2 - 2c\modelparams_{0, \unprot}^\top \Sigma^{\frac{1}{2}}\xi.
    \end{align}
    Define $r=\modelparams_{0, \unprot}^\top \Sigma^{\frac{1}{2}}\xi$. 
    We can choose $\xi$ in the direction of the main singular vector of $\modelparams_{0, \unprot}^\top \Sigma^{\frac{1}{2}}$, such that
    \begin{align}
        |r| &= \lVert \modelparams_{0, \unprot}^\top \Sigma^{\frac{1}{2}}\xi \rVert 
        = \lVert \modelparams_{0, \unprot}^\top \Sigma^{\frac{1}{2}}\rVert \lVert \xi \rVert 
        = \sqrt{\tau_{\Bar \alpha}} \lVert \modelparams_{0, \unprot}^\top \Sigma^{\frac{1}{2}}\rVert,
    \end{align}
    and $\text{sign}(r)= \text{sign}(-c)$, which gives us the maximum as
    \begin{align}
    c^2 + \tau_{\Bar{\alpha}} \lVert \modelparams_{0, \unprot}^\top \Sigma^{\frac{1}{2}}\rVert^2 - 2c\sqrt{\tau_{\Bar{\alpha}}} \lVert \modelparams_{0, \unprot}^\top \Sigma^{\frac{1}{2}}\rVert.
    \end{align}
The expression can be compactly written as
\begin{equation}
\label{eq:proof-advLoss}
    \left(|y-\modelparams_0^\top \inferredtruth| + \sqrt{\tau_{\Bar{\alpha}}} \lVert \modelparams_{0}^\top \unprotmatrix\rVert\right)^2.
\end{equation}

This formulation is convex in $\modelparams_0$ since it is a monotonically increasing function of a weighted sum 
of convex functions. Finally, the full risk \eqref{eq:risk_stealthy_linear_regression} thus is another weighted sum of convex functions in $\theta_0$ and $\theta_1$.
\end{proof}

Similarly, the risk under attack for logistic regression can be shown to be convex.

\begin{proposition}
\label{prop:stealthyrisk_log}
    Consider the logistic regression model \eqref{eq:logisticregression} with anomaly detector \eqref{eq:energydetector}. 
    The risk 
    \eqref{eq:switched_risk} under stealthy adversarial attacks with attack probability $\gamma$, is 

\begin{equation}
\begin{split}
\risk(\modelparams_0, \modelparams_1) = 
\gamma &\expect\left[\log \left( 1 + e^{-y\modelparams_0^\top g(x_p) + \sqrt{\tau_{\Bar{\alpha}}}\|\modelparams_0^\top\unprotmatrix\|} \right) \right] \\
+ (1-\gamma)&\expect[\detector\log\left(1 + e^{-y\theta_1^\top x}\right) \\ &+ (1-\detector)\log\left(1 + e^{-y \theta_0^\top x}\right)].
\end{split}
\label{eq:risk_stealthy_logistic_regression}
\end{equation}
This risk is convex in the model parameters $\modelparams$.
\end{proposition}




\begin{proof}
The proof uses the same main ideas as the one for Proposition \ref{prop:stealthyrisk_lin}
First, the risk is divided into cases depending on whether an attack is present, and the second term follows directly from plugging in the logistic loss.
Now, consider the first term, where an attack is present. Note that by monotonicity of the logarithm and exponential, maximizing the logistic loss is equivalent to maximizing
\begin{equation}
    -\target\modelparams_0^\top z = -y \modelparams_{0,p}^\top\truth_\prot - y\modelparams_{0, \unprot}^\top z_u.
\end{equation}
Letting $\xi= \Sigma^{-\frac{1}{2}}(z_u - \inferredtruth_u)$, we get
\begin{equation}
    \max_{\xi \leq \sqrt{\tau_{\Bar{\alpha}}}} -y\modelparams_{0}\inferredtruth + \modelparams_0^\top\unprotmatrix \xi.
\end{equation}
By choosing $\xi$ in the proper direction, we obtain the maximum loss
\begin{equation}
    \log\left(1+e^{-y\modelparams_0^\top\inferredtruth + \sqrt{\tau_{\Bar{\alpha}}} \|\modelparams_0^\top \unprotmatrix\|}\right),
\end{equation}
and the desired formulation is obtained.
The convexity follows straight-forward from the convexity of the logistic loss and norms.
\end{proof}

It can be noted that as $\gamma \rightarrow 0$ in \eqref{eq:risk_stealthy_linear_regression} and \eqref{eq:risk_stealthy_logistic_regression}, the risk minimizing $\theta_0^\top = \begin{bmatrix}
    \theta_p^\top & \theta_u^\top
\end{bmatrix}$ sets $\theta_u=0$, as this minimizes $\|\theta_0^\top M\|$ and the rank of $\theta_0^\top g(x_p)$ allows $\theta_u$ to be treated as a free variable. 


\section{Method}
\label{section:method}
The adverserial training is achieved by minimizing the empirical expectation 
$   \risk_n(\modelparams_0, \modelparams_1) = \expect_n[\loss_{\modelparams_0, \modelparams_1}(\observation, y)] = \frac{1}{n}\sum_{i=1}^n \loss_{\modelparams_0, \modelparams_1}(\observation^{(i)}, y^{(i)})$ in
\eqref{eq:risk_stealthy_linear_regression} or \eqref{eq:risk_stealthy_logistic_regression}, with constraint \eqref{eq:lemma}, using the clean training data. We use stochastic gradient descent with momentum and decreasing step size. The method requires specifying an upper bound on false alarm $\Bar{\alpha}$, and on the attack probability, denoted $\gamma_\text{train}$.





\section{Numerical results}
\label{section:numerical_results}
For comparison, we use two different baselines. First, a \textit{fully secure} model denoted  $f_{\theta_s, \theta_s}$, with the restriction $\theta_s = \begin{bmatrix}
    \theta^\top_p & 0_u
\end{bmatrix}^\top$, so that the unprotected features are discarded and the model is secure w.r.t. any detector $\delta$ and itself by Definition \ref{def:secure}. Second, a \textit{standard switched} model denoted $f_{\theta_n, \theta_s}$ with $\theta_n = \arg \min_{\theta_0}\expect[\loss_{\theta_0}(x,y)]$, i.e. trained to minimize the risk on clean data, with constraint \eqref{eq:lemma}, which we denote $f_{\theta_n, \theta_s}$. 

All models are evaluated with respect to their (empirical) risk under data that is under  stealthy adversarial attacks with probability $\gamma_{\text{test}}$. The data is split 50-50 into a train and a test set, and we evaluate the method on both synthetic and real world data. 


\subsection{Synthetic data}

We evaluate our methodology for linear regression where $d=4$ and the label is given by $y= \sum_{i=1}^4 x_i$. The distribution $p(x)$ is a zero mean Gaussian with $\var[x_i]=1$ for all $i$ and $\var[x_1,x_3] = \var[x_2, x_4] = 0.8$ (remaining covariances are zero). We use $n=2000$ training samples. 

Fig.~\ref{fig:synthetic} illustrate the results. First, stealthy attacks are clearly detrimental to the standard model since no attacks are detected! Second, when $\gamma_{\text{train}} = \gamma_\text{test}=0$, 
the proposed switched model has the same performance but converges to that of  the fully secure model as $\gamma_\text{test}$ increases. Between these points it shows an improved performance. Third, even as the assumed upper bound is incorrect, $\gamma_\text{train} < \gamma_\text{train}$, there is still a range in which the risk is lower than the secure model and is consistently lower than the standard model. 

Simulations for other parameters of the Gaussian, and for logistic regression, show the same overall behaviour with variation only in convergence rate, and plots for these are therefore excluded from this paper.

\begin{figure}
\begin{subfigure}{0.22\textwidth}
    \includegraphics[width=\linewidth]{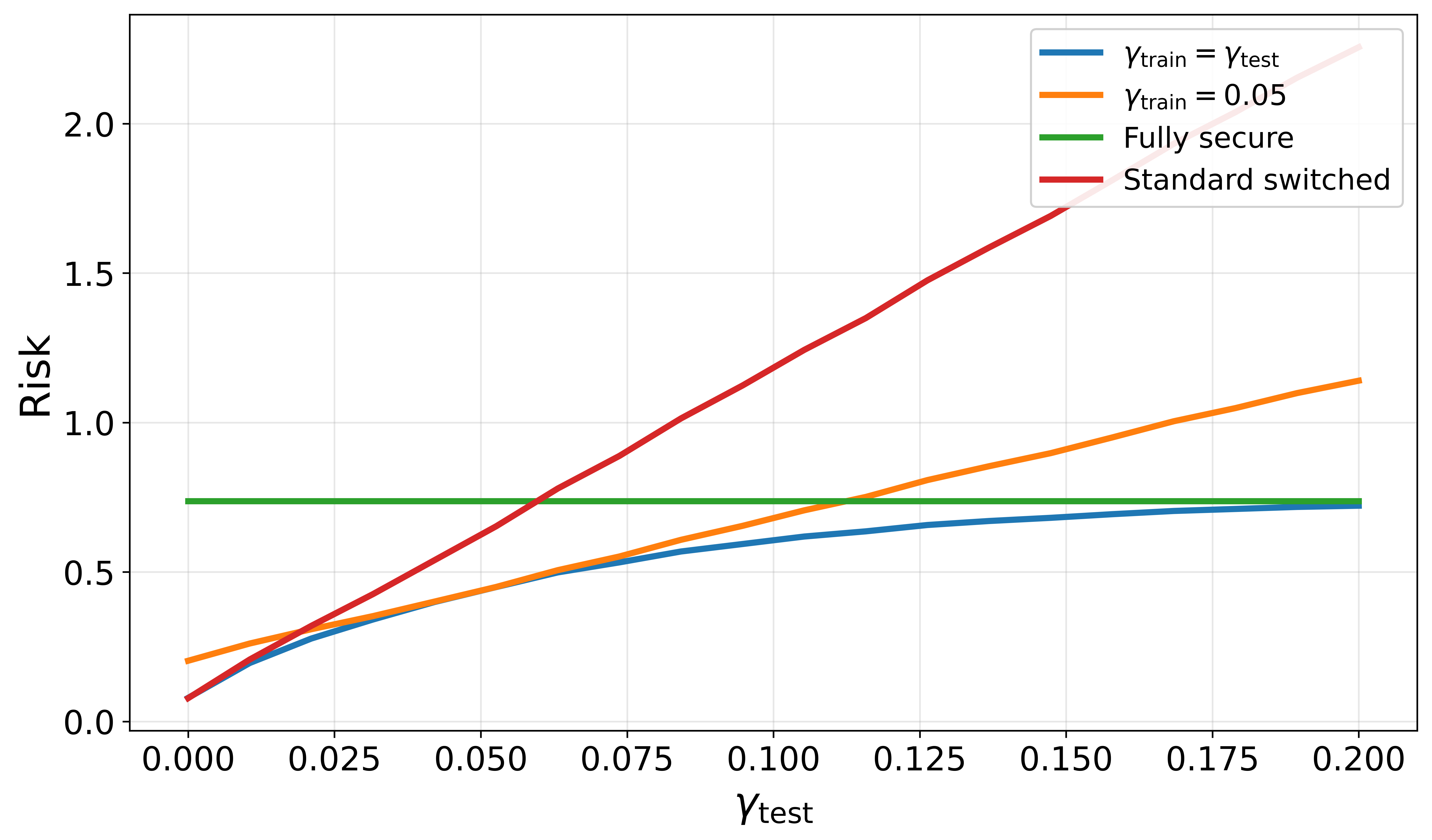}
    \caption{}
    \label{fig:linear_synthetic}
\end{subfigure}
\hfill
    \begin{subfigure}{0.22\textwidth}
    \includegraphics[width=\linewidth]{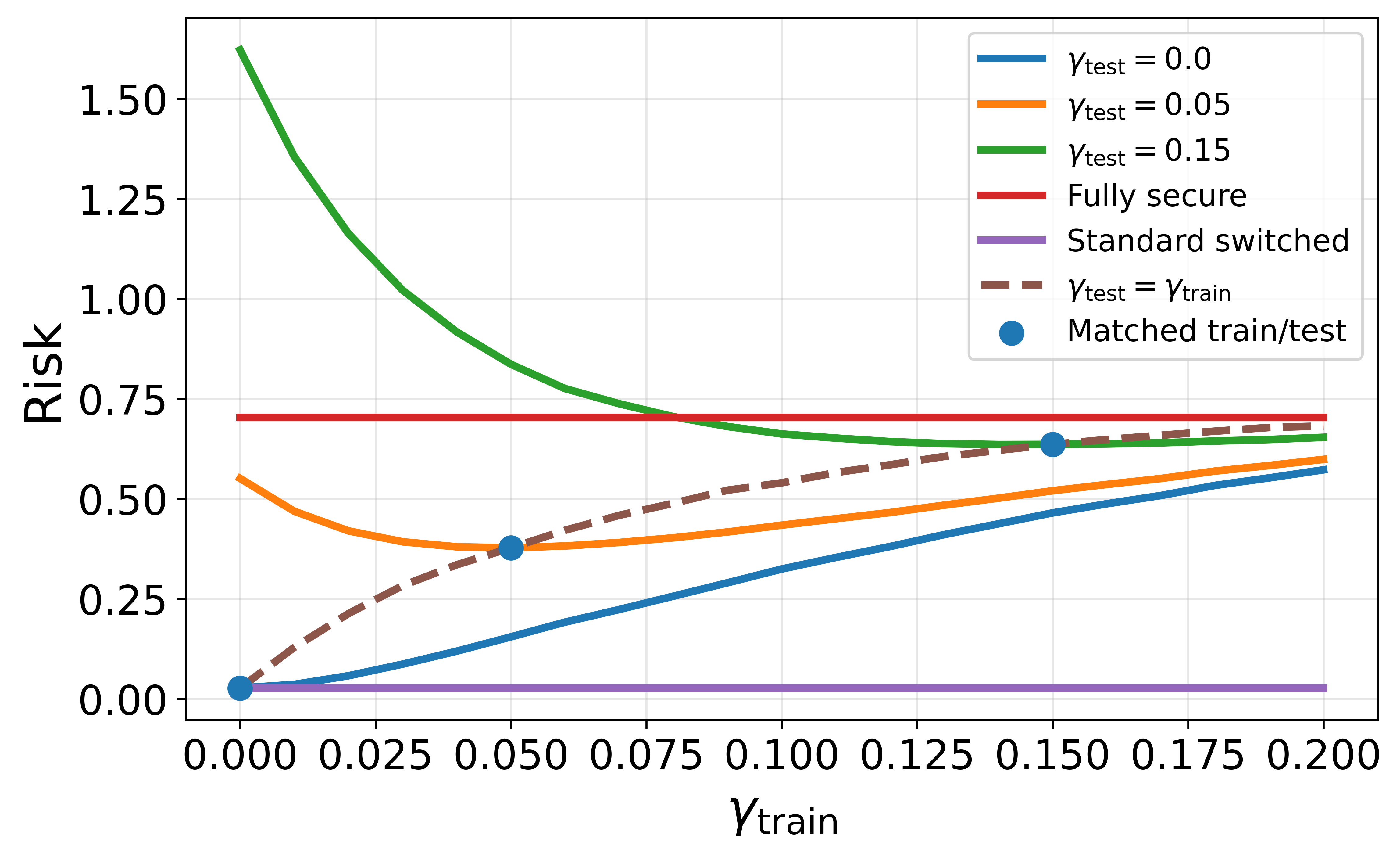}
    \caption{}
    \label{fig:misspecified}
\end{subfigure}
\caption{(a) Risk versus attack probability $\gamma_{\text{test}}$, including misspecified $\gamma_{\text{train}}=0.05$. (b) Risk versus $\gamma_{\text{train}}$, where red dots mark when the model is correctly specified and the dotted line is the correctly specified model from (a). Both (a) and (b) also show  the standard switched model $f_{\theta_n, \theta_s}$, and the fully secure model $f_{\theta_s, \theta_s}$. Both cases set $\bar{\alpha}= 0.01$ yielding an actual false alarm rate of 1.3\%.}
\label{fig:synthetic}
\end{figure}

\subsection{Real world data}
We further evaluate our training scheme on a real word Credit Card Fraud Detection data set \cite{frauddata} with $n=569k$ samples. The feature `amount' is treated as protected and 21 features are treated as unprotected (7 features removed due to high correlation with other features). The results are shown in Fig. \ref{fig:real}, where we see that the method also works for complex real word data with a more significant number of features.

\begin{figure}
\begin{subfigure}{0.22\textwidth}
    \includegraphics[width=\linewidth]{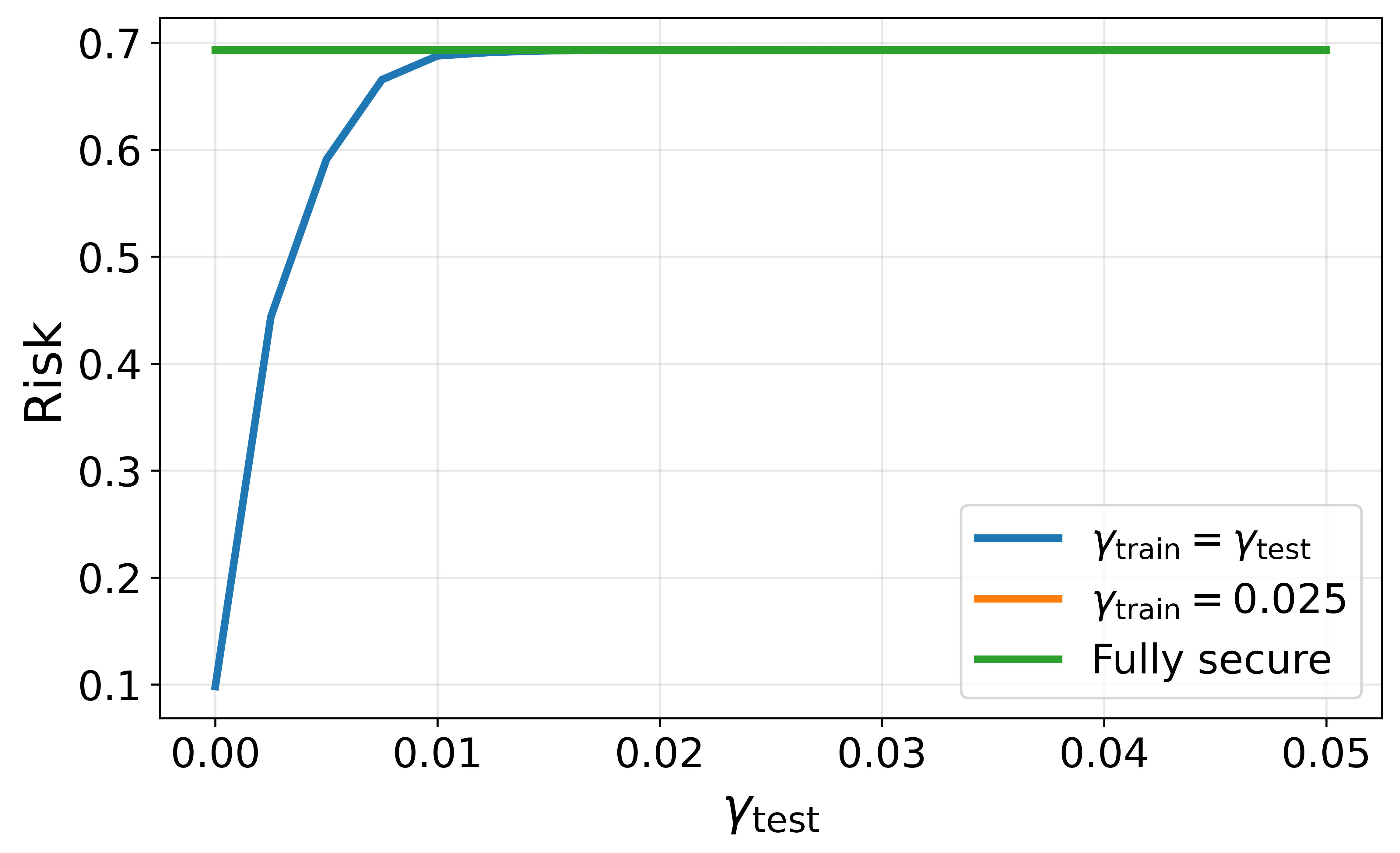}
    \caption{}
    \label{fig:zoomed}
\end{subfigure}
\hfill
    \begin{subfigure}{0.22\textwidth}
    \includegraphics[width=\linewidth]{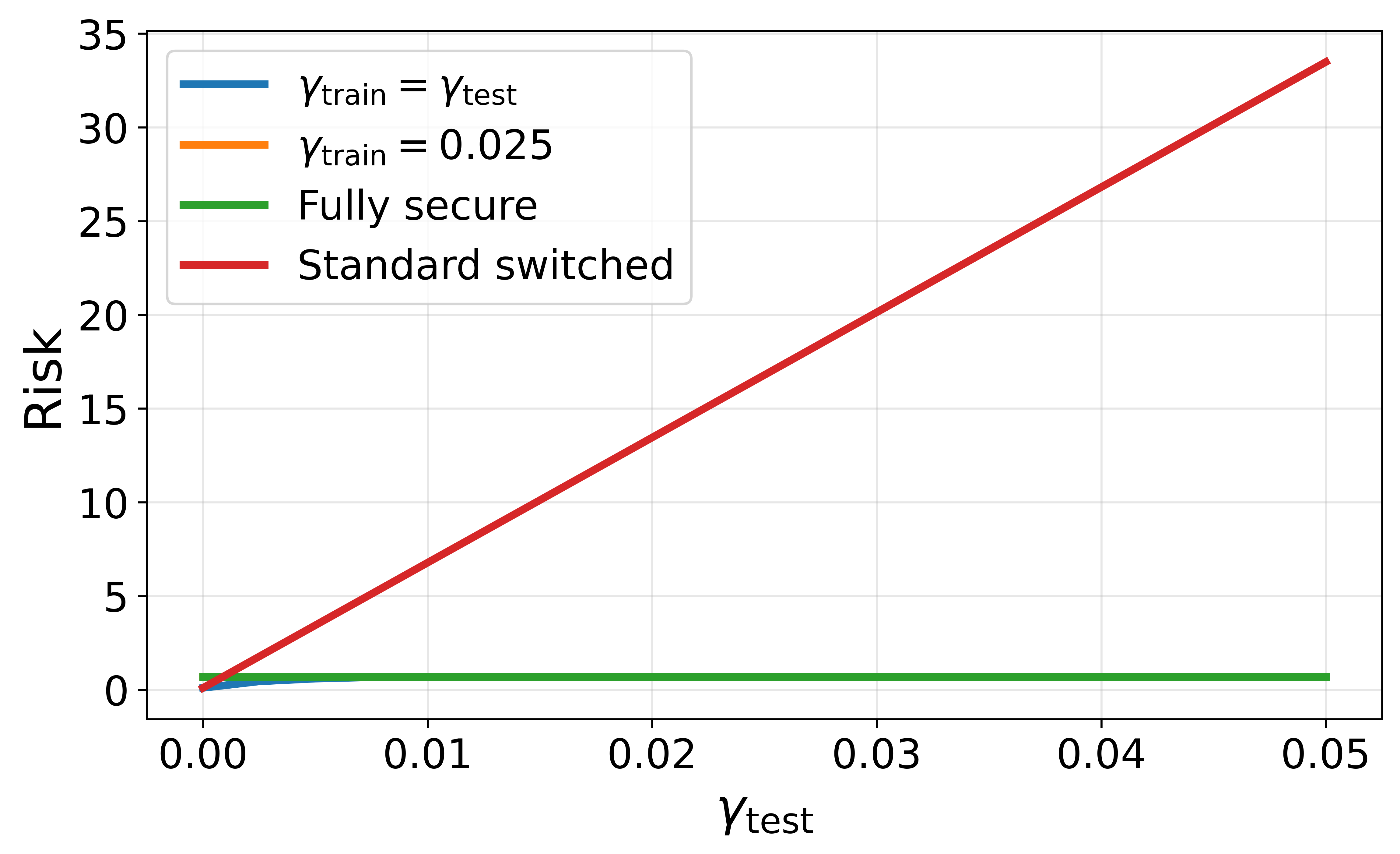}
    \caption{}
    \label{fig:naive}
\end{subfigure}
\caption{(a) Risk versus attack probability $\gamma_{\text{test}}$. (b) Same as (a) but also includes the standard switched model. In this case we set the bound $\bar{\alpha}$ yielding an actual false alarm rate of 0.01\%.}
\label{fig:real}
\end{figure}



\section{Conclusion}
\label{section:Discussion}
We have shown that using an attack detector and aligning the recovery model with the nominal model, adversarial attacks are forced to be stealthy. A convex risk for this setup is derived for linear and logistic regression, and numerical results support the optimality of the resulting model. Future work include extension to more complex models and detectors, extension to neural networks, and further investigating the relation between our method, adversarial training, and regularisation.

\bibliography{bibliography}
\bibliographystyle{IEEEtran}






\end{document}